\documentclass[shortAfour,sageh,times]{sagej}
\usepackage{moreverb,url}

\usepackage{hyperref}{}{}

\newcommand\BibTeX{{\rmfamily B\kern-.05em \textsc{i\kern-.025em b}\kern-.08em
T\kern-.1667em\lower.7ex\hbox{E}\kern-.125emX}}

\newcommand{\cost}{\mathrm{cost}}
\usepackage{graphicx}
\usepackage{subcaption}
\usepackage[ruled,vlined,linesnumbered]{algorithm2e}
\usepackage{times,amsthm}
\providecommand{\norm}[1]{\left\lVert#1\right\rVert}
\renewcommand{\paragraph}[1]{\medskip\noindent\textbf{#1}}
\newcommand{\carab}{\textsc{Mean-Coreset}}
\newcommand{\OPT}{\mathrm{OPT}}
\newcommand{\coreset}{\textsc{Mean-Coreset}}
\newcommand{\coresetb}{\textsc{Kabsch-Coreset}}
\newcommand{\ccoreset}{\textsc{Streaming-Coreset}}
\newcommand{\instream}{stream}
\newcommand{\br}[1]{\left\{#1\right\}}                            
\newcommand{\eps}{\varepsilon}
\newcommand{\each}{\emph{\bf each} }
\newcommand{\REAL}{\ensuremath{\mathbb{R}}}

\newtheorem{lemma}{Lemma}
\newtheorem{theorem}{Theorem}
\newtheorem{definition}[theorem]{Definition}

\usepackage{amsmath,amssymb,latexsym,amsfonts}
\begin{document}
\runninghead{Nasser et al.}

\title{Coresets for Kinematic Data: \\From Theorems to Real-Time Systems}

\author{Soliman Nasser\affilnum{1}, Ibrahim Jubran\affilnum{1} and Dan Feldman\affilnum{1}}

\affiliation{\affilnum{1}Robotics and Big Data Lab, University of Haifa, IL}

\corrauth{Soliman Nasser, 
Robotics and Big Data Lab,
University of Haifa, IL.}

\email{soliman.nasser7@gmail.com}

\begin{abstract}
A coreset (or core-set) of a dataset is its semantic compression with respect to a set of queries,
such that querying the (small) coreset \emph{provably} yields an approximate answer to querying the original (full) dataset. In the last decade, coresets provided breakthroughs in theoretical computer science for approximation algorithms, and more recently, in the machine learning community for learning ``Big data". However, we are not aware of real-time systems that compute coresets in a rate of dozens of frames per second.

In this paper we suggest a framework to turn theorems to such systems using coresets.

We begin with a proof of independent interest, that \emph{any} set of $n$ matrices in $\REAL^{d\times d}$ whose sum is $S$, has a positively weighted subset whose sum has the same center of mass (mean) and orientation (left+right singular vectors) as $S$, and consists of $O(dr)$ matrices (independent of $n$), where $r\leq d$ is the rank of $S$. 
We provide an algorithm that computes this (core) set in one pass over possibly infinite stream of matrices in $d^{O(1)}$ time per matrix insertion.

By maintaining such a coreset for kinematic (moving) set of $n$ points, we can run pose-estimation algorithms, such as Kabsch or PnP, on the small coresets, instead of the $n$ points, in real-time using weak devices,
while obtaining the same results. This enabled us to implement a low-cost ($<\$100$) IoT wireless system that tracks a toy (and harmless) quadcopter which guides guests to a desired room (in a hospital, mall, hotel, museum, etc.) with no help of additional human or remote controller.

See the supp. material for a video that demonstrates our tracking system, including this ``Guardian Angel" application.

We hope that our framework will encourage researchers outside the theoretical community to design and use coresets in future systems and papers. To this end, we provide extensive experimental results on both synthetic and real data, as well as a link to the open code of our system and algorithms.
\end{abstract}

\keywords{Core-sets, Pose Estimation, Caratheodory, Localization}

\maketitle

\begin{figure}[hbtp]
\centering
\includegraphics[scale=0.4]{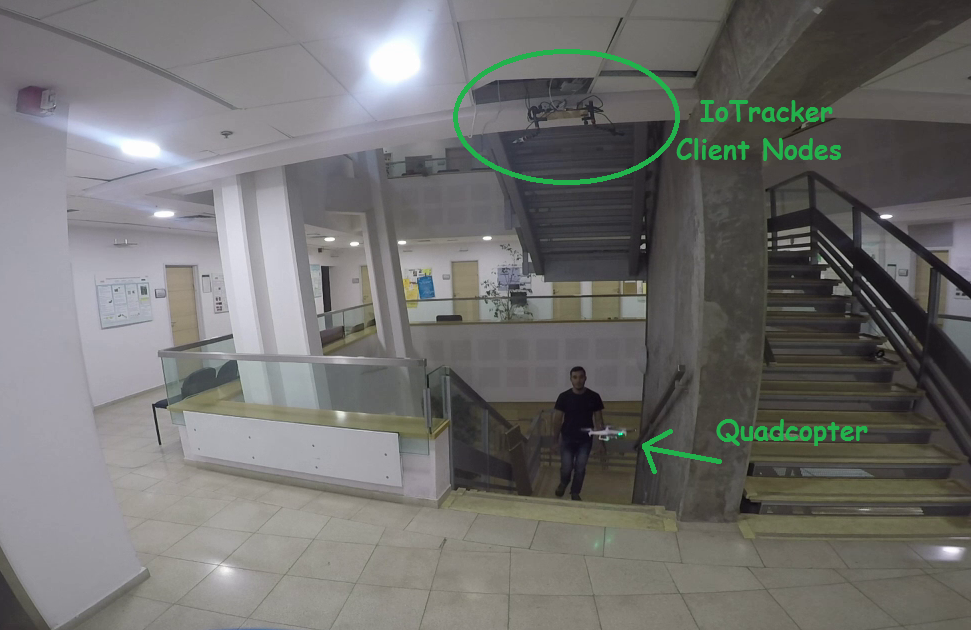}
\caption{"Guardian Angle" system. A safe and low-cost quadcopter autonomously leads a guest to its destination. Video: \url{https://vimeo.com/244656298} \label{ga-system}}
\end{figure}

\section{Introduction and Motivation}\label{sec:mot}
Coresets is a powerful technique for data reduction that was originally used to improve the running time of algorithms in computational geometry (e.g.~\cite{aga1,aga2,HP01,feldman2007ptas,sur,CzuSoh07a,Phillips16}. Later, coresets were designed for obtaining the first PTAS or LTAS (polynomial/linear time approximation schemes) for more classic and graph problems in theoretical computer science~\cite{FraSoh05,czumaj2005approximating,FrahlingIS08,BuriolFLS07}. More recently, coresets appear in machine learning  conferences~\cite{feldmanmik,feldman2011scalable,tsang2005core,lucic2016strong,bachem2016approximate,lucic2015tradeoffs,bachem2015coresets,huggins2016coresets,rosman2014coresets,reddi2015communication}.
with robotics~\cite{feldmanmik,sung2012trajectory,feldman2013idiary,rusprivate,rosman2014coresets,feldman2011scalable,bachem2015coresets,lucic2016strong,bachem2016approximate,reddi2015communication} and image~\cite{feigin2011high,feldman2013learning,alexandroni2016coresets} applications.

This paper has three goals:
(i) introduce coreset for the system community, and show how their theory can be applied in real-time systems, and not only in the context of machine learning or theoretical computer science.\\
(ii) suggest novel coresets for kinematic data, where the motivation is not to directly improve the running time of an algorithm, but to select a small subset of moving points that can be tracked and processed during the movement of the set in the next observed frames.\\
(iii) provide a wireless and low-cost tracking system, \textit{IoTracker}, based on mini-computers (``Internet of things") that is based on coresets.

To obtain goal (i) we suggest a simple but powerful and generic coreset that approximates the center of mass of a set of points, using a sparse distribution on a small subset of the input points. While this mean coreset has many other applications, to obtain goal (ii) we use it to design a novel coreset for pose-estimation based on the alignment between two paired sets. Computing the orientation of a moving robot or a rigid body is a fundamental question in SLAM (Simultaneous Localization And Mapping) and computer vision;see references in~\cite{stanway2015rotation}.

For example, we prove that the result of running the classic Kabsch algorithm that optimally solved this problem under natural assumptions, would yield the same result when applied on the coreset. This holds even after the input set (including its coreset) is translated and rotated in space, without the need of recomputing the coreset. We prove that the coreset has constant size (independent of the number of input tracked points), for every given input set.

Although we proved the correctness of the coreset for the Kabsch algorithm, by its properties we expect that it will hold for many other pose-estimation algorithms. Even when it is not, the coreset may be used in practice with any other pose estimation algorithm.

To demonstrate goal (iii) we install our tracking system in a university (and soon in a mall) and implement a ``Guardian angel" application, that to our knowledge implement for the first time fictional systems such as Skycall~\cite{skycall}: a safe and low-cost quadcopter leads a guest to its destination room based on pre-programmed routs, and the walking speed of the human. The main challenge was to control a sensors-less quadcopter in few dozens of frames per seconds using weak mini-computers. Unlike existing popular videos (e.g.~\cite{skycall}), in our video the quadcopter is autonomous in the sense that there is no hidden remote controller or another human in the loop, see \cite{coresetVideo17}.

This ``Guardian angel" was our main motivation and inspiration for designing the coreset in this paper.

\textbf{This paper is not }about suggesting the best algorithm for pose-estimation, the best tracking system, or about localization of quadcopters. As stated above, our goals are to show the process cycle from deep theorems in computational geometry, as the Caratehorody Theorem, to real-time and practical systems that use coresets. Nevertheless, we are not aware of similar coresets for pose estimation of kinematic data, or low-cost wireless tracking systems that can be used for hovering of a very unstable quadcopter in dozens of frames per second.

\label{sec:ourContrib}

\section{Related Work and Comparison\label{related}}

\paragraph{Pose Estimation. }
The pose estimation problem is also called the alignment problem, since given two paired point sets, $P$ and $Q$, the task is to find the Euclidean motion that brings $P$ into the best possible alignment with $Q$. We focus on the case where this alignment is the translation and rotation of $P$ that minimizes the sum of squared distances to the point of $Q$. For $|P| = |Q| = n$ points $\in \REAL^d$,  the optimal translation $\mu^*$ is simply the mean of $Q$ minus the means of $P$, each can be computed in $O(nd)$ time. Computing the optimal rotation $R^*$ (Wahba's Problem~\cite{wahba65}) can be obtained by the Kabsch algorithm~\cite{kabsch1976solution} in $O(nd^2)$ time; see Theorem~\ref{kabsch}

\paragraph{In the PnP problem} the observed set $Q$ is computed from a single camera, resulting in a set $Q$ of $n$ \emph{lines} which makes the problem NP-hard unlike the problems that are discussed in this paper. Indeed, exact solution for the PnP problem are known only for the case $n\leq 4$, and no provable approximation are known even for $n>4$ and sum of squared distances.
The Kabsch Coreset in this paper may be used to improve the running time of common PnP heuristics by running them on the coreset. Unlike their usage for Kabsch Algorithm, the theoretical guarantees of the coreset would no longer hold.

Sort of coreset of 4 point for PnP was suggested in~\cite{lepetit2009epnp}. However this set is not a subset of the input and unlike our Kabsch coreset, running PnP on the coreset would necessarily yield an approximated solution for the original set.

\paragraph{ICP.} In the previous paragraphs we assumed that the matching between $P$ and $Q$ is given.
The standard and popular solution for solving the matching and pose-estimation problems is called Iterative
Closest Point (ICP) proposed by Besl and McKay~\cite{besl1992method}; see~\cite{wang2015comparisons} and references therein. 
Variations and speed-ups can be found in~\cite{friedman1977algorithm, zhang1994iterative, pulli2000surface}

\paragraph{Faster and robust matching using coresets.} Our Kabsch coreset as the Kabsch algorithm, assumes that the matching between the points in the registered and observed frame is given. Matching is a much harder problem than e.g. the Kabsch algorithm (that can be solved in $O(n)$ time) in the sense that we have $n!$ permutations. Nevertheless, mean coreset can reduce both the running time and the robustness of the matching process. For example, in ICP, each point in $P\subseteq\REAL^3$ is assigned to its nearest neighbour (NN) in $Q$ which take $O(|P|\cdot |Q|)$ time. Using our Kabsch coreset for $P$ the running time reduced to $O(|Q|)$. This also implies that NN matching can be replaced in existing applications by a slower but better algorithm (e.g. cost-flow~\cite{ahuja1993network}) that will run on the small coreset. This will improve the matching step of ICP, without increasing the existing running times. Such an improvement is relevant even for non-kinematic (single) pair $P$ and $Q$ of points.

\textbf{Table ~\ref{table-label}} concludes the time comparison of solving each step of the localization problem with/without using coresets. 
The first row of the table represents the case where the matching has already been computed, and what is left to compute is the optimal rotation between the two sets of points.
The second row represents step (2) of the localization problem, where the matching needs to be computed given the rotation. In this case, a perfect matching between a set of size $k$ to a set of size $m$, can be achieved, according to \cite{bipmatching}, in $O(\sqrt{m+k}mk\log(m+k)$ time. Without using a coreset, the size of both sets is $n$. When using a coreset, the size of $P$ is reduced to $rd$, although the size of $Q$ remains $n$.
The last row represents a case where we need to compute the matching between two sets of points and the correct alignment is not given. In this case there are $n!$ possible permutations of the original set, each with its own optimal rotation. Using the coreset, the number of permutations reduces to roughly $(rd)!$ since it suffices to match correctly only the coreset points.

\begin{table}[]
\centering
\caption{\small Time comparison. All the numbers written in the table are in $O$ notation and represent time complexity. See more details in supplementary material.}

\label{table-label}
\begin{tabular}{c|c|c|}
\cline{2-3}
                                                                                                & \begin{tabular}[c]{@{}c@{}}Without coreset\\ $|P| = n , |Q| = n$\end{tabular} & \begin{tabular}[c]{@{}c@{}}Using coreset\\ $|P| = rd$\end{tabular}    \\ \hline
\multicolumn{1}{|c|}{\begin{tabular}[c]{@{}c@{}}With matching,\\ without rotation\end{tabular}} & \begin{tabular}[c]{@{}c@{}} \\ $nd^2$\end{tabular} & \begin{tabular}[c]{@{}c@{}}$|Q| = rd$\\ $d^3r$\end{tabular}           \\ \hline
\multicolumn{1}{|c|}{\begin{tabular}[c]{@{}c@{}}Without matching,\\ with rotation\end{tabular}} & \begin{tabular}[c]{@{}c@{}} \\ $n^{2.5}\log(n)$\end{tabular}  & \begin{tabular}[c]{@{}c@{}}$|Q| = n$\\$n^{1.5}dr\log(n)$\end{tabular} \\ \hline
\multicolumn{1}{|c|}{Noisy matching}                                                            & \begin{tabular}[c]{@{}c@{}} \\
$nd^2\cdot (n!)$\end{tabular}   & \begin{tabular}[c]{@{}c@{}}$|Q| = rd$\\ $(dr)!$\end{tabular}           \\ \hline
\end{tabular}
\end{table}

\paragraph{Relation to other coresets.} Unlike existing coresets, the coresets in this paper are exact and have no approximation error $\eps$.
This allows us to obtain the optimal solution for the problem. In other coresets, even for the related SVD problem, it is not clear how the approximation error will affect the output rotation matrix that is returned by the Kabsch algorithm. Since our mean coreset does not introduce any error, it can be used in any applications that aims to compute any functions $f(AA^T)=f(\sum_i a_ia_i^T)$, since it preserves the sum $\sum_i a_ia_i^T$.

An exception is the coreset $S'=DV^T$ for a matrix $A=UDV^T$, where $UDV^T$ is the SVD of $A$, or $S'=QR$ where $QR$ is the $QR$ decomposition of $A$ (Gram–Schmidt). In both cases, $\norm{Ax}_2=\norm{S'x}$ for every $x\in\REAL^d$ and there is no approximation error. However, in this case the rows of $S'$ are not a scaled subset of the input rows. Besides numerical and interpretation issues, we cannot use this coreset for kinematic data since we do not have a subset of markers to track over time or between frames.

\textbf{Coreset for sum of $1$-rank positive definite matrices }of size $O(d/\eps^2)$ were described e.g. in~\cite{summ}; see references therein. Our mean coreset is larger but implies such an exact result, and is more general (sum of any $d\times d$ matrices).

\section{Warm up: Mean Coreset\label{MeanCoreset}}
Given a set $P$ of $n$ points ($d$-dimensional vectors), our basic suggested tool is a small weighted subset of $P$, that we call \emph{mean coreset}, whose weighted mean is \emph{exactly} the same as the mean of the original set. In general, we can simply take the mean of $P$ as a coreset of size $1$. However, we require that the coreset will be a \emph{subset} of the input set $P$. Moreover, we require that the vector of the multiplicative weights will be a sparse distribution over $P$, i.e., a positive vector with an average entry of $1$. There are at least three reasons for using this coreset definition in practice, especially for real-time kinematic/tracking systems:\\
\textbf{$(i)$ Numerical stability: } Every $d$ linearly independent points in $P$ span their mean. However, this coreset yields huge positive and negative coefficients that cancelled each other and resulted in high numerical error. Our requirement that the coreset weights will have positive weights whose average is $1$ -- make these phenomena disappear in practice.
 \\
 \textbf{$(ii)$ Efficiency: }A small coreset allows us to compute the mean of a kinematic (moving) set of points faster, by computing the mean of the small coreset in each frame, instead of the complete set of points. This also reduces the time and probability of failure of other tasks such as matching points between frames. This is explained in Section~\ref{sec:ourContrib}.
 \\
 \textbf{$(iii)$ Kinematic Tracking: } In the next sections we track the orientation of an object (robot or a set of vectors), by tracking a kinematic representative set (coreset) of markers on it over time. This is impossible when the coreset is not a subset of input points.

\begin{definition}[\label{cordef}Mean coreset]
A \emph{distribution vector} $u=(u_1,\cdots,u_n)$ is a vector whose entries are non-negative and sum to one.
A \emph{weighted set} is a pair $(P,u)$ where $P=\br{p_1,\cdots,p_n}$ is an ordered set in $\REAL^d$, and $u$ is a distribution vector of length $|P|$.

A weighted set $(S,w)$ is a \emph{mean coreset} for the weighted set $(P,u)$, if $S\subseteq P$ and their weighted mean is the same, i.e.,
\[
\sum_{i=1}^n u_i p_i=\sum_{j=1}^{|S|} w_js_j,
\]
where $S=\br{s_1,\cdots,s_{|S|}}$. The \emph{cardinality} of the mean coreset $(S,w)$ is $|S|$.
\end{definition}

Of course $P$ is a trivial coreset of $P$. However, the coreset $S$ is efficient if its size $|S|=\big|\br{i\mid w_i>0}\big|$ is much smaller than $|P|=n$. This is related to the Caratheodory Theorem~\cite{cara} from computational geometry, that states that any convex combination of a set $P$ of points (in particular, its mean) is a convex combination of at most $d+1$ points in $P$.

We first suggest an inefficient construction in Algorithm~\ref{corealg} to obtain a mean coreset of only $d+1$ points, i.e., independent of $n$, for a set of $n$ points. This is based on the proof of the Caratheodory Theorem which we give for completeness, and takes $O(n^2d^2)$ time, which is impractical for the applications in this paper.

\textbf{Overview of Algorithm~\ref{corealg} and its correctness.}
The input is a weighted set $(P,w)$ whose points are denoted by $P=\br{p_1,\cdots,p_n}$. We assume $n>d+1$, otherwise $(S,u)=(P,w)$ is the desired coreset. Hence, the $n-1>d$ points $p_2-p_1$, $p_3-p_1,p_4-p_1,\ldots$ must be linearly dependent. This implies that there are reals $v_2,\cdots,v_n$, which are not all zeros, such that
\begin{equation}\label{eq00}
\sum_{i=2}^n v_i (p_i-p_1)=0.
\end{equation}
These reals are computed in Line~\ref{l11} by solving system of linear equations. This step dominates the running time of the algorithm and takes $O(nd^2)$ time using SVD. The definition
\begin{equation}\label{uudef}
v_1=-\sum_{i=2}^n v_i
\end{equation}
in Line~\ref{u1}, guarantees that
\begin{equation}\label{eqq}
\begin{split}
&\sum_{i=1}^n v_i p_i=v_1p_1+\sum_{i=2}^n v_i p_i=-\sum_{i=2}^n v_i p_1+\sum_{i=2}^n v_i (p_i-p_1)\\
&=\sum_{i=2}^n v_i (p_i-p_1)=0,
\end{split}
\end{equation}
where the second equality is by~\eqref{uudef}, and the last is by~\eqref{eq00}.
Hence, for every $\alpha\in\REAL$, the weighted mean of $P$ is
\begin{equation}\label{sum}
\sum_{i=1}^n u_ip_i=\sum_{i=1}^n u_ip_i+\alpha\sum_{i=1}^n v_i p_i=  \sum_{i=1}^n \left(u_i-\alpha v_i\right) p_i.
\end{equation}
The definition of $\alpha$ in Line~\ref{alp} guarantees that $\alpha v_i=u_i$ for some $i^*\in[n]$, and that $\alpha\leq 1$ for every $i\in[n]$. Hence, the set $S$ that is defined in Line~\ref{Sdef} contains at most $n-1$ points, and its set of weights $\br{u_i-\alpha v_i}$ is non-negative. The sum of weights is
\[
\sum_{i=1}^n (u_i-\alpha v_i)=\sum_{i=1}^nu_i-\alpha\cdot \sum_{i=1}^n v_i=1,
\]
where the last equality hold by~\eqref{uudef} and since $u$ is a distribution vector. This and~\eqref{sum} proves that $S$ is a mean coreset as in Definition~\ref{cordef} of size $n-1$. In Line~\ref{eight} we repeat this process recursively until there are only $d+1$ points left in $S$. For $O(n)$ iterations the overall time is thus $O(n^2d^2)$.

\begin{algorithm}[th]
\DontPrintSemicolon
{\begin{tabbing}
\textbf{Input:\quad} \= A weighted set $(P,u)$ of $n$ points in $\REAL^{d}$.\\
\textbf{Output:} \> A mean coreset $(S,w)$ for $P$ \\\>of cardinality $|S|\leq d+1$; see Definition~\ref{cordef}.
\end{tabbing}}

\If{$|P|\leq d+1$}
{\Return $(P,u)$}
{\For {$\each$ $i\in\br{2,\cdots,n}$}{
Set $a_i\gets p_i - p_1$
}
Set $A$ to be the $d\times (n-1)$ matrix $[ a_2 | \cdots | a_{n}]$.\\
Compute $v=(v_2,\cdots,v_{n})^T\neq 0$ such that $Av=0$. \label{l11}\\
Set $\displaystyle v_1 \gets -\sum_{i=2}^{n}  v_i$\label{u1}\\
Set $\displaystyle \alpha \gets \min\br{\frac{u_i}{v_i } \mid i\in [n] \text{ and }  v_i> 0}$ \label{alp}\\

Set $w_i\gets (u_i-\alpha v_i)$ for every $i\in[n]$ s.t. $w_i>0$. \label{Sdef}\\
Set $S\gets \br{p_i\mid w_i>0 \text{ and } i\in[n]}$\label{Sdeff}\\
\If {$|S|>d+1$ \label{eight}}{
Set $(S,w)\gets \carab(S,w)$}
}\Return $(S,w)$
\caption{$\carab(P,u)$}\label{corealg}
\end{algorithm}

The correctness of the following lemma follows mainly by the Caratheodory Theorem~\cite{cara} from computational geometry.
\begin{lemma}\label{lemma}
Let $P=\br{p_1,\cdots,p_n}$ be a set of $n>d+1$ points in $\REAL^d$.
Let $S$ be the output of a call to $\coreset(P)$; see Algorithm~\ref{corealg}. Then $S$ is a mean coreset of $P$.
This takes $(n-(d+1))\cdot O(nd^2)=O(n^2d^2)$ time.
\end{lemma}

\newcommand{\algstr}{\textsc{Streaming-Coresets}}
\newcommand{\coresetsize}{CoresetSize}
\begin{algorithm}[th]
\DontPrintSemicolon
{\begin{tabbing}
\textbf{Input:\quad} \=A (possibly infinite) $\instream$ of points in $\REAL^{d}$.\\
\textbf{Output:} \>A mean coreset $(S,w)$ of cardinality $|S|\leq d+1$ \\\>for the first $n$ points in $\instream$, for every $n\geq1$.
\end{tabbing}}
Set $S\gets\emptyset$\label{A1}\\
\While{ $\instream$ is not empty\label{A2}}{
    Set $n\gets n+1$\label{A3}\\
 $q\gets$ read the $n$th point from $\instream$\label{A4}\\
    Set $P\gets S\cup \br{q}$ \\\tcc{$S$ is the first $|S|$ points in $P$}\label{l5}\label{A5}
    Set $u_i\gets\label{A6}
    \begin{cases}
       \frac{w_i(n-1)}{n} & 1\leq i \leq |S|\\
       \frac{1}{n} &   i= |S|+1
    \end{cases}$\\
    $(S,w)\gets \carab(P,u)$\label{A7}\\
    \textbf{Output} $(S,w)$}\label{l8}\label{A8}
\caption{$\ccoreset(\instream)$\label{strcaraalg}}
\end{algorithm}

We then use the fact that our mean coresets are composable ~\cite{indyk2014composable,mirrokni2015randomized,aghamolaei2015diversity}: a union of coresets can be merged and reduced again recursively. To reduce the running time of Algorithm~\ref{corealg} we thus run it only on the first $d+2$ points of $P$, reduce it to a coreset of $d+1$ points in $O(d^3)$ time using a single iteration, and repeat for each of the remaining points.

\textbf{Overview of Algorithm~\ref{strcaraalg} and its correctness.}
We denote $[n]=\br{1,\cdots,n}$ for every integer $n\geq1$. In Line~\ref{A1} we initialize the coreset $S$.
In Line~\ref{A2} we begin to read the points in the (possibly infinite) input stream of points.
In Line~\ref{A3} we update this counter $n$, and in Line~\ref{A4} we read the $n$th point from the stream.
The set $P$ in Line~\ref{A5} is the union of the coreset for the points read till now with the new $n$th point $q$.

In Line~\ref{A6} we define a distribution vector $u$ such that the weighted set $(P,u)$ has the same mean as the mean of the $n$ points $q_1,\cdots,q_n$ that were read till now. The intuition is that the new points represents a faction of $1/n$ from the $n$ points seen so far, but $S$ (the rest of points in $P$) represents $(n-1)/n$ input points. If the $i$th point in $S$ has a weight $w_i$, it means that it represents a fraction of $w_i$ from $S$, i.e., fraction of $w_i(n-1)/n$ from all the data. Indeed, the mean of the $n$ read points $q_1,\cdots,q_n$ and $P$ is the same,
\begin{equation}\label{eq44}
\begin{split}
\frac{1}{n}\sum_{i=1}^n q_i&=\frac{1}{n}\sum_{i=1}^m q_i+\frac{1}{n}\sum_{q\in Q}q\\
&=\frac{m}{n}\sum_{i=1}^{|S|} w_is_i+\sum_{q\in Q}\frac{1}{n}\cdot q
=\sum_{i=1}^{|P|}u_ip_i.
\end{split}
\end{equation}
Also, $u$ is a distribution vector since
\[
\sum_{i=1}^n u_i=\frac{1}{n}+\sum_{i=1}^{|S|}\frac{w_i(n-1)}{n}=\frac{1}{n}+\frac{n-1}{n}=1,
\]
where the second equality is since $w$ is a distribution vector by induction.

In Line~\ref{A7}, we compute a mean coreset $(S,w)$ for $(P,u)$. Since $|P|=d+2$, by Lemma~\ref{lemma} this takes $O(d^3)$, and by~\eqref{eq44} $(S,w)$ is also the mean coreset for the $n$ points read till now.
In Line~\ref{A8} we output $(S,w)$ and repeat for the next point. The required memory is dominated by the set $P$ of $2d+2$ points. We conclude with the following theorem.

\begin{theorem}\label{thmm}
Let $\instream$ be a procedure that outputs a new point in $\REAL^d$ after each call.
A call to $\ccoreset(\instream)$ outputs a mean coreset of cardinality $d+1$ for the first $n$ points in $\instream$, for every $n\geq1$. This takes $O(d^3)$ time for each point update, overall of $O(nd^3)$ time and using at most $d+2$ points in memory.
\end{theorem}

\subsection{Example Applications}
\paragraph{Sum coreset for matrices.} Theorem~\ref{thmm} implies that we can compute the sum of $n$ matrices in $\REAL^{d\times d}$ using a weighted subset of $d^2$ matrices, simply by concatenating the entries of each matrix to a vector in $\REAL^{d^2}$. In Section~\ref{sec:CoresetDef} we reduce this size for the case where we are only interested in the left+right singular vectors of the matrix. This reduction is theoretically small but allowed us to reduce the number of IR markers by more than half in the third paragraph of Subsection~\ref{autuQuad} which was critical to the IR tracking version of our system.

\paragraph{Coreset for SVD.} Let $A\in\REAL^{n\times d}$. Our mean coreset implies that there is a matrix $S$ that consists of $O(d^2)$ scaled rows in $A$ such that for every $x\in\REAL^d$,
$
\norm{Ax}_2=\norm{Sx}_2.
$
This is since
\[
\norm{Ax}_2^2=(Ax)^T(Ax)=x^TA^TAx=x^T(\sum_{i=1}^n a_ia_i^T) x.
\]
The rightmost term can be computed using a sum coreset for matrices as defined above.

\paragraph{Coreset for Linear Regression.} In the case of linear regression, we are also given a vector $b\in\REAL^n$ and wish to compute a matrix $S$ of $O(d^2)$ weighted rows from $A$, and a vector $v$ of the same size, such that for every $x\in \REAL^d$ we have
\[
\norm{Ax-b}_2=\norm{Sx-v}_2.
\]
This can be obtained by replacing $A$ with $[A \mid -b]$ in the previous example.

\paragraph{Streaming and Distributed computation. }Theorem~\ref{thmm} implies that we can compute the above coresets also for possibly infinite streaming set of row vectors or matrices. Similarly, using $m$ machines the running time can be reduced by a factor of $m$ by sending the $i$th point in the stream to the $(i \mod m)$th machine.

\section{Application for Kinematic data: \\Kabsch Coreset}
\label{sec:CoresetDef}
To track a kinematic set of points in $\REAL^3$ (e.g. markers or visual features on a rigid body), we need to define its initial (zero) position, the registered set $P$, and compare it to the observed set $Q$ in the current time or frame. The difference (translation and rotation) between $P$ and $Q$ tells us the current position of the set. Using the Maximum Likelihood approach, and the common assumption of Gaussian noise (which has physical justification), the optimal solution is the translation and rotation of $P=\br{p_1,\cdots,p_m}$ that minimize the sum of squared distances to the corresponding points in $Q=\br{q_1,\cdots,q_n}$,
\[
\cost(P,Q,R):=\sum_{i=1}^n \norm{p_i-Rq_i}^2,
\]
This is known as Wahba's Problem~\cite{wahba65}. We denote this minimum by
\[
\OPT(P,Q):=\min_R \cost(P,Q,R)=\cost(P,Q,R^*).
\]

\paragraph{Tracking translation.} Easy calculations show that the optimal translation as defined above is the mean (center of mass) of $Q$. This mean can be maintained by tracking only the small mean coreset of $Q$ over time as defined in Section~\ref{MeanCoreset}, even without knowing the matching between $P$ and $Q$. 

In this section we thus focus on the more challenging problem of computing the rotation $R$ that minimizes the sum of squared distances between the points of $P$ and $RQ$.

The Kabsch algorithm~\cite{kabsch1976solution} suggests the following simple but provably optimal solution for Wahba's problem. Let $UDV^T$ be a Singular Value Decomposition (SVD) of the matrix $P^TQ$. That is, $UDV^T=P^TQ$, $U^TU=V^TV=I$, and $D\in\REAL^{d\times d}$ is a diagonal matrix whose entries are non-increasing.  In addition, assume that $\det(U)\det(V)=1$, otherwise invert the signs of one of the columns of $V$. Note that $D$ is unique but there might be more than one such factorization.
\begin{theorem}[~\cite{kabsch1976solution}\label{kabsch}]
The matrix $R^*=VU^T$ minimizes $\cost(P,Q,R)$ over every rotation matrix $R$, i.e.,
$
\OPT(P,Q)=\cost(P,Q,R^*).
$
\end{theorem}

We now suggest a coreset (sparse distribution) for this problem.
\begin{definition}[Kabsch Coreset\label{defcore}]
Let $w\in S^n$ be a distribution that defines the matrices $\tilde{P}=\br{\sqrt{w_i}p_i \mid w_i>0, 1\leq i\leq n}$ and $\tilde{Q}=\br{\sqrt{w_i}q_i \mid w_i>0, 1\leq i\leq n}$.
Then $w$ is a \emph{Kabsch coreset} for the pair $(P,Q)$ if for every pair of rotation matrices $A,B\in\REAL^{d\times d}$ and every pair of vectors $\mu,\nu\in\REAL^d$ the following holds:
A rotation matrix $\tilde{R}$ that minimizes
$\cost(\tilde{P}A+\mu, \tilde{Q}B+\nu,R)$ over every rotation matrix $R$, is also optimal for $(PA+\mu,QB+\nu)$, i.e.,
\[
\OPT(PA+\mu,QB+\nu)=\cost(PA+\mu,QB+\nu,\tilde{R}).
\]
\end{definition}

This implies that we can use the same coreset even if the set $Q$ is translated or rotated over time. Such a coreset is efficient if it is also small (i.e. the distribution vector $w$ is sparse).
\newcommand{\RemovePoint}{\textsc{DeletePoint}}
\begin{algorithm}[th]
{\begin{tabbing}
\textbf{Input:\quad} \=A pair of matrices $P,Q\in\REAL^{n\times d}$.\\
\textbf{Output:} \>A sparse Kabsch coreset $w=(w_1,\cdots,w_n)$ \\ \>
for $(P,Q)$; see Definition~\ref{defcore}.
\end{tabbing}}
Set $d'\gets r(d-1)$ where $r$ is the rank of $P^TQ$ \\
Set $UDV^T\gets$ an SVD of $P^TQ$ \\
\For {\textbf{each} $i\in[n]$}{
Set $p_i$ and $q_i$ to be the $i$th row of $P$ and $Q$, respectively. \\
Set $m_i\in\REAL^{d'}$ as the entries of $U^Tp^T_iq_iV$, excluding its diagonal and last $d-r$ rows. \label{fivef}\\
}
$\instream\gets \br{m_1,m_2,\cdots,m_n}$\label{sixf} \\\
$(S,w')\gets \ccoreset(\instream)$; see Algorithm~\ref{strcaraalg}\label{sevenf}\\
$(w_1,\cdots,w_n)\gets$ the weights of $w'$ that correspond to $S$ padded with zero entries.\label{eightf}\\ 
\Return $w$\\
\caption{\coresetb($P,Q$)\label{caraalg}}
\end{algorithm}

Recall that $UDV^T$ is the SVD of $P^TQ$, and let $r$ denote the rank of $P^TQ$, i.e., number of non-zero entries in the diagonal of $D$.
Let $D_r\in\REAL^{d\times d}$ denote the diagonal matrix whose diagonal is $1$ in its first $r$ entries, and $0$ otherwise.

\begin{lemma}\label{lem}
Let $R=GF^T$ be a rotation matrix, such that $F$ and $G$ are orthogonal matrices, and $GD_rF^T=VD_rU^T$.
then $R$ is an optimal rotation, i.e.,
$$\OPT(P,Q)=\cost(P,Q,R).$$
Moreover, the matrix $VD_rU^T$ is unique and independent of the chosen Singular Value Decomposition $UDV^T$ of $P^TQ$.
\end{lemma}\begin{proof}
\newcommand{\tr}{\mathrm{Tr}}
It is easy to prove that $R$ is optimal, if
\begin{equation}\label{e4}
\tr(RP^TQ)=\tr(D);
\end{equation}
see~\cite{kjer2010evaluation} for details.
Indeed, the trace of the matrix $RP^TQ$ is
\begin{align}
\nonumber\tr(RP^TQ)&=\tr(RUDV^T)
=\tr(GF^T(UDV^T))\\
\label{e1}&=\tr(GD_rF^T\cdot UDV^T)\\
\label{e11}&\quad+\tr(G(I-D_r)F^T\cdot UDV^T).
\end{align}
Term \eqref{e1} equals
\begin{equation}\label{e2}
\begin{split}
&\tr(GD_rF^T\cdot UDV^T)=\tr(VD_rU^T\cdot UDV^T) \\
&=\tr(VDV^T))=\tr(DV^TV)=\tr(D),
\end{split}
\end{equation}
where the last equality holds since the trace is invariant under cyclic permutations.
Term~\eqref{e11} equals
\[
\begin{split}
&\tr(G(I-D_r)F^T\cdot UDV^T) \\
&=\tr(G(I-D_r)F^T\cdot (D_rU^T)^T DV^T)\\
&=\tr(G(I-D_r)F^T\cdot (V^TGD_rF^T)^T DV^T) \\
&=\tr(G(I-D_r)F^T\cdot FD_r^TG^TV\cdot DV^T)\\
&=\tr(G\cdot (I-D_r)D_r \cdot G^TV\cdot DV^T)
=0,
\end{split}
\]
where the last equality follows since the matrix $(I-D_r)D_r$ has only zero entries.
Plugging the last equality and~\eqref{e2}
in~\eqref{e1} yields $\tr(RP^TQ)=\tr(D)$.
Using this and~\eqref{e4} we have that $R$ is optimal.

For the uniqueness of the matrix $VD_rU^T$,
observe that for $N=P^TQ=UDV^T$ we have
\begin{equation}\label{e5}
(N^TN)^{1/2}(N)^{+}
=(VDV^T)(VD^{+}U^T)
=VD_rU^T.
\end{equation}
Here, a squared root $X^{1/2}$ for a matrix $X$ is a matrix such that $(X^{1/2})^2=X$, and $X^{+}$ denote the pseudo inverse of $X$.
Let $FEG^T$ be an SVD of $N$. Similarly to~\eqref{e5}, $(N^TN)^{1/2}(N)^{+}=GD_rF^T$.

Since $N^TN=VD^2V^T$ is a positive-semidefinite matrix, it has a unique square root.
Since the pseudo inverse of a matrix is also unique, we conclude that $(N^TN)^{1/2}(N)^{+}$ is unique, and thus $VD_rU^T=GD_rF^T$.
\end{proof}

\textbf{Overview of Algorithm~\ref{caraalg}. }The input is a pair $(P,Q)$ of $n\times d$ matrices that represent two paired set of points in $\REAL^d$. To obtain an object's pose, we need to apply the Kabsch algorithm on the matrix $P^TQ=\sum_ip_i^Tq_i$; see Theorem~\ref{kabsch}. Algorithm~\ref{caraalg} outputs a sparse weight vector $w=(w_1,\cdots,w_n)$ such that the summation $P^TQ$ equals to the weighted sum $\sum_iw_ip_i^Tq_i$ of at most $r(d-1)+1$ matrices, where $w$ is a \emph{Kabsch-coreset} as in Definition~\ref{defcore}.

This is done by by choosing $w$ such that
\begin{equation}\label{EE}
E=U^T\left(\sum_i w_ip_iq_i^T\right)V=\sum_i w_i (U^Tp_iq_iV)
\end{equation}
is a diagonal matrix. In this case, the rotation matrix of the pairs $(\sqrt{w_i}p_i, \sqrt{w_i}q_i)_{i=1}^n$ and $(P,Q)$ will be the same by Theorem~\ref{kabsch}. By letting $m_i=(U^Tp_iq_iV)$ we need to have the sum $\sum_{i=1}^n m_i$ by a weighted subset of the same sum.

This vector $m_i$ is computed in Line~\ref{fivef}. In Line~\ref{sixf} we compute a mean coreset $(S',w')$ using Algorithm~\ref{strcaraalg} for the $n$ vectors $m_1,\cdots,m_n$. Since the mean coreset contains only the non-zero weights with their corresponding points, in Line~\ref{eightf} we translate the $|S'|=O(d^2)$ weights in $w'$ to the sparse vector $w$: if $s_i$ is the $i$th point in $S'$ and $s_i=m_j$, then $w_j=w'_i$. Theorem~\ref{thmm} then guarantees that~\eqref{EE} holds as desired.

We now prove main theorem of this section.
\begin{theorem}\label{mainthm}
Let $P,Q\in\REAL^{n\times d}$ be a pair of matrices.
Let $r$ denote the rank of the matrix $P$.
Then a call to the procedure $\coresetb(P,Q)$ returns a Kabsch-coreset $w$ of sparsity at most $r(d-1)+1$ for $(P,Q)$ in $O(nd^4)$ time; see Definition~\ref{defcore}.
\end{theorem}
\begin{proof}
Since $(S,w')$ is a mean coreset for $m_1,\cdots,m_n$ we have that $w$ is a distribution of sparsity at most $r(d-1)+1$, such that
\begin{equation}\label{eq10}
  E=U^T\left(\sum_i p_iq_i^T\right)V=U^T\left(\sum_i w_ip_iq_i^T\right)V
\end{equation}
is diagonal and consists of at most $r$ non-zero entries. Here $p_i$ and $q_i$ are columns vectors which represent the $i$th row of $P$ and $Q$ respectively.
Let $\br{\sqrt{w_i}p_i\mid w_i>0}$ and $\br{\sqrt{w_i}q_i\mid w_i>0}$ be the rows of $\tilde{P}$ and $\tilde{Q}$ respectively.
Let $FEG^T$ be an SVD of $A^T\tilde{P}^T\tilde{Q}B$ such that $\det(F)\det(G)=1$, and let $\tilde{R}=GF^T$ be an optimal rotation of this pair; see Theorem~\ref{kabsch}. We need to prove that $$\OPT(PA+\mu,QB+\nu)=\cost(PA+\mu,QB+\nu,\tilde{R}).$$
We assume without loss of generality that $\mu=\nu=0$, since translating the pair of matrices does not change the optimal rotation between them~\cite{kjer2010evaluation}.

By~\eqref{eq10}, $UEV^T$ is an SVD of $\tilde{P}^T\tilde{Q}$, and thus $A^TUEV^TB$ is an SVD of $A^T\tilde{P}^T\tilde{Q}B$. Replacing $P$ and $Q$ with $\tilde{P}A$ and $\tilde{Q}B$ respectively in Lemma~\ref{lem} we have that $GD_rF^T=B^TVD_rU^TA$.
Note that since $UDV^T$ is an SVD of $P^TQ$, we have that $A^TUDV^TB$ is an SVD of $A^TP^TQB$.
Using this in Lemma~\ref{lem} with $PA$ and $QB$ instead of $P$ and $Q$ respectively yields that $\tilde{R}=GF^T$ is an optimal rotation for the pair $(PA,QB)$ as desired, i.e.,
\[
\OPT(PA,QB)=\cost(PA,QB,\tilde{R}).
\]

\end{proof}

\section{From Theory to Real Time Tracking System\label{realtime}}
While our coresets are small and optimal, they come with a price: unlike random sampling which takes sub-linear time
to compute (without going over the markers), computing our  coreset takes the same time as solving the
 pose estimation problem on the same frame. Hence, we use the following pair of parallel threads.

The first thread, which we run at $1$ to $3$ FPS (frames per second), gets a snapshot (frame) of markers and computes the coreset for this frame. This includes marker identification, matching problem, and then
computing the actual coreset for the original set of markers $P$ and the observed set $Q$.
The second thread, which calculates the object's pose, runs every frame. In our low-cost tracking system (see Section~\ref{sec:Experiments}) it handles $30$ FPS. This is by using the last computed coreset on the new frames, until the first thread computes a new coreset for a later frame. The assumption of this model is that, for frames that are close to each other in time, the translation and rotation of the observed set of markers will be similar to the translation and rotation of the set $Q$ in the previous frame, up to a small
error. Theorem~\ref{kabsch} guarantees that the coreset for the first frame will still be the same for the new frame.

\section{Experimental Results}

\label{sec:Experiments}
We run the following types of experiments:

\subsection{Synthetic data} We constructed a set $P$ of $n$ randomly and uniformly sampled points in $\REAL^3$, a rotation matrix $R\in\REAL^{3\times 3}$ and a translation vector $t\in\REAL^d$ from a uniform distribution. We defined $Q = P\cdot R + t$ and aims to reconstruct $R$ and $t$ using the following methods: (i) Calculate the optimal rotation matrix and optimal translation vector from $P$ and $Q$, as described in Section~\ref{sec:CoresetDef}, (ii) Compute the same from the  Kabsch-coreset (see algorithm~\ref{A3}) of size $r\cdot (d-1)+1 = 7$ (where $r=d=3$) and the Mean-coreset (see algorithm~\ref{A1}) of size $d+1=4$, (iii) Uniform sampling of two sets of corresponding points from $P$ and $Q$, one of size $7$ and the second of size $4$, and compute $R$ and $t$ from these sets.

\paragraph{Non-noisy data.} Here we generated data as described above for $100$ iterations, where the set $P = \{p_1,p_2,...,p_{300}\}$ consisted of $300$ randomly sampled points. Each point $p_i \in [0,3000]^3$, $t \in [0,3000]^3$ and $R$ was randomly selected among all valid 3D rotation matrices. We then compared methods (i) and (ii) where the coreset was computed in the first iteration only, and used throughout all other iterations. The results are shown in Fig.~\ref{non-noisy-data}. As proven in Section~\ref{sec:CoresetDef}, the two methods yielded similar results since the data is non-noisy. Surprisingly, the coreset error is sometimes even lower than the error of the optimal method probably since the coreset reduces numerical errors; see beginning of Section~\ref{MeanCoreset}.

\begin{figure}[hbtp]
\centering
\includegraphics[scale=0.137]{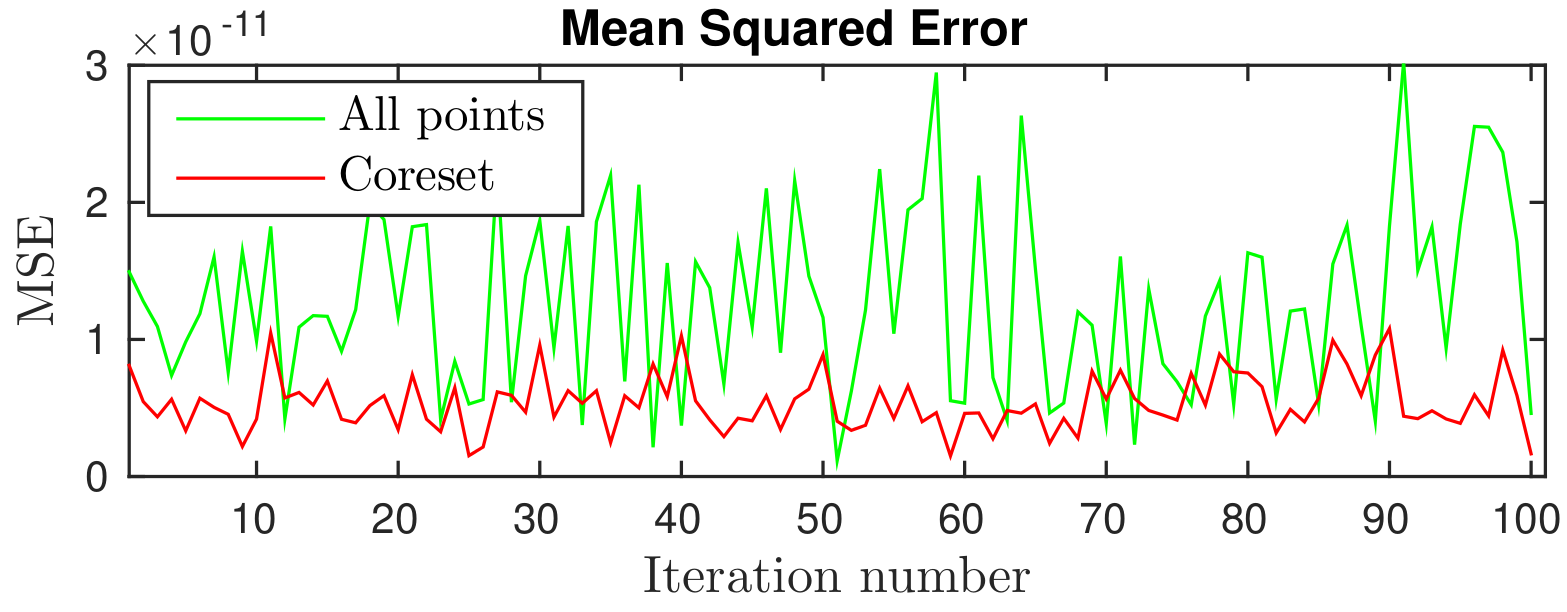}
\caption{\small Comparing the results of methods (i) and (ii). The $X-axis$ represents the number of iterations. The $Y-axis$ represents the Mean Squared Errors between the two sets after applying the optimal poses obtained from each of the two methods. \label{non-noisy-data}}
\end{figure}

\paragraph{Noisy data. } Here our goal was to test the coresets in the presence of noise. We generated a set $P = \{p_1,p_2,...,p_{100}\} \in R^{100 \times 3}$ of $100$ randomly sampled points. Each point $p_i \in [0,1000]^3$, $t$ is a random vector in $[0,1000]^3$ and $R$ was randomly selected among all valid 3D rotation matrices. We then computed the set $Q = P\cdot R + t + m\cdot B$, where $B \in R^{100 \times 3}$ consists of random and uniform noise in the range $[0,100]$, and $m$ is the magnitude of the noise. This test compares the error produced by methods (i)--(iii) while increasing the value of $m$ for multiple iterations. The coreset was recomputed every $x$ iterations and the random points were also resampled every $x$ iterations, where $x$ is the calculation cycle. The results are shown in Fig.~\ref{noisy-data}, the first graph shows the results for $x = 20$, the second graph shows the results for $x = 300$ and the third graph shows the results for $x = \infty$ (i.e. computed only once). The results show a steady increase in the error of method (iii). Our coreset's error steadily increases until a new coreset is recalculated, at that point the coreset error realigned with the error of method (i) as expected, resulting in stiff decreases that are seen in the graphs. Moreover, the coreset error converges to the error of the random sampling in the third graph (as expected) since the coreset is not recomputed while the noise magnitude becomes larger, in this case the coreset points do not outperform a random sample of the points.

\begin{figure}[hbtp]
\centering
\includegraphics[scale=0.6]{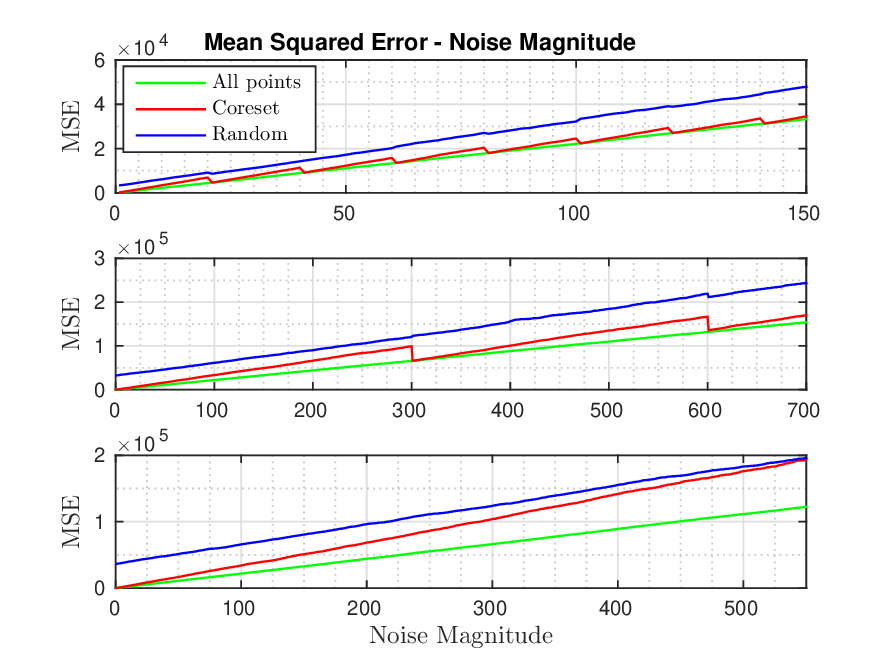}
\caption{\small Comparing the results of methods (i), (ii) and (iii). The $X-axis$ represents the noise magnitude $m$. The $Y-axis$ represents the MSE between the two sets after applying the optimal poses obtained from each of the methods. \label{noisy-data}}
\end{figure}

\paragraph{Running Time. }To evaluate the running time of our algorithms, we apply them on random data using a laptop with an Intel Core i7-4710HQ CPU @ 2.50GHz processor. We compared the calculation time of the pose estimation on a coreset vs. the full set. This test consists of two cases: a) Using an increasing number of points while  maintaining a constant dimension,  b) Using a constant number of points of different dimensions. The results are shown in Figures \ref{fig-time-a} and \ref{fig-time-b} respectively.
The test corresponds to the first row of Table~\ref{table-label}. Fig. \ref{fig-time-a} shows that when the coreset size of algorithm~\ref{A3} is larger than the number $n$ of points, the computation time is roughly identical, and as $n$ reaches beyond $dr = O(d^2)$, the computation time using the full set of points continues to grow linearly with $n$ ($O(nd^2)$), while the computation time using the coreset, which is dominated by the computation of the optimal rotation, ceases to increase since it is independent of $n$ ($d^3r$ = $O(d^4)$). Fig. \ref{fig-time-b} shows that the coreset indeed yields smaller computation times compared to the full set of points when the dimension $d < \sqrt{n}$, and both yield roughly the same computation time as $d$ reaches $\sqrt{n}$ and beyond.

\begin{figure}
\begin{subfigure}[h]{0.49\textwidth}
		\centering
		\includegraphics[scale=0.55]{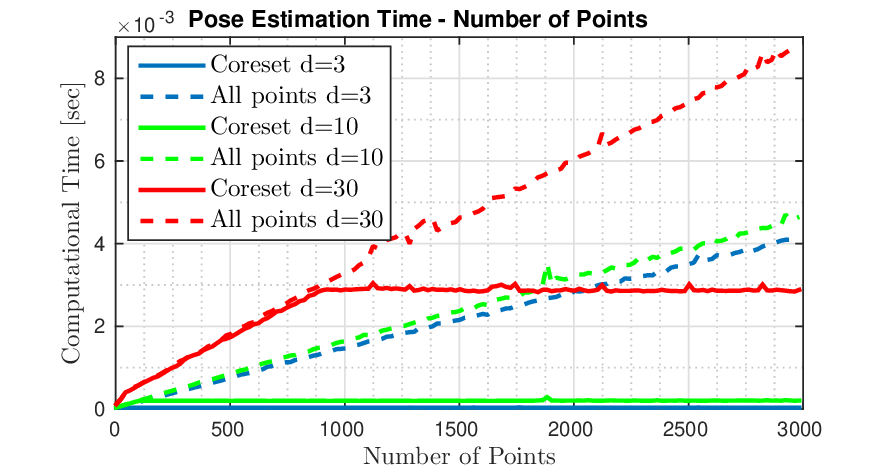}
		\centering\caption{\label{fig-time-a}}
	\end{subfigure}
	\begin{subfigure}[h]{0.49\textwidth}
		\centering\centering		
		\includegraphics[scale=0.55]{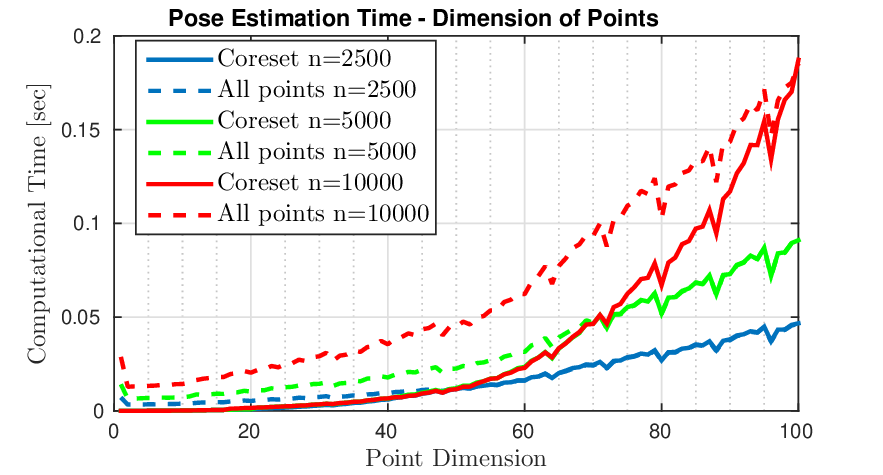}
		\caption{\label{fig-time-b}}
	\end{subfigure}

\caption{\small Time comparison between calculating the orientation of $n$ points of dimension $d$ given a previously calculated coreset versus using all $n$ points. The Y-axis of both graphs represents the time needed to obtain the orientation. In~\ref{fig-time-a}, the X-axis represents the number of points while in~\ref{fig-time-b} the X-axis represents the points dimension.\label{compgraphs}}
\end{figure}

\subsection{Multi-camera Wireless Low-cost Tracking System}
We developed a wireless and low-cost home-made indoor tracking system ($<\$100$) based on web-cams, IoT mini-computers and the algorithms in this paper to compensate the weak hardware. The system consists of distributed ``client nodes" (one or more) and one ``server node". Each client node contains 2 components: (A) a mini-computer, Odroid U3 ($<\$30$) and (B) a pair of standard web-cams (SONY PSEye, $<\$5$). The server node consists only of a mini-computer. The server node runs the two threads discussed in Section~\ref{realtime}.

\subsubsection*{Autonomous quadcopter.\label{autuQuad}}
We used our tracking system to compute the $6DoF$ of the quadcopter and send control commands accordingly after reverse engineering its communication protocol. We compared the orientation error of the quadcopter using our coreset as compared to uniform sampling of the IR or visual markers on the quadcopter. 

In both tests, the coreset was computed every $x$ frames, the random points were also sampled every $x$ frames, where $x$ is the calculation cycle time. The chosen weighted points were used for the next $x$ frames, and then a new Kabsch-Coreset of size $r(d-1)+1=5$ was computed by Algorithm \ref{A3}, where $d=3$ and $r=2$ as the features on the quadcopter are roughly in a planar configuration.\\
See Section~\ref{realtime} and the video in the supplementary material for demonstrations and results.

\begin{figure}[h]
\centering
\includegraphics[scale=0.6]{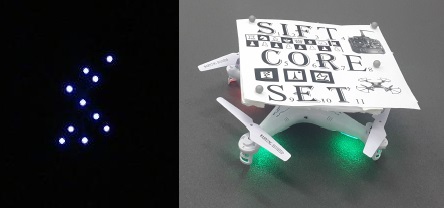}
\caption{\textbf{(left)} 10 IR markers as captured by the web-camera with the IR filter. \textbf{(right)} {A toy micro-quadcopter with a planar pattern (printed text and other featurs) placed on top}.\label{fig2}}
\end{figure}

\paragraph{Infra-Red (IR) Tracking.}
Following the common approaches used by the commercial tracking systems, we used IR markers for tracking.
We placed $10$ Infra-red LEDs on the quadcopter and modified the web-cams' lenses to let only infrared spectrum rays pass, see Fig.~\ref{fig2}. We could not place more than $10$ LEDs on such a micro quadcopter because of over-weight problem and short battery life.
Since the sensorless quadcopter requires a stream of at least $30$ control commands per second in order to hover and not crash, we apply the Kabsch algorithm only on a selected a subset of 5 points. Our experiments showed that even for such small numbers, choosing the right subset is crucial for a stable system.

The system computes the 3D location of each LED using triangulation.  Afterwards, it uses Algorithm~\ref{A3} to compute a Kabsch-Coreset of size $r(d-1)+1 = 5$ from the 3D locations, where $d=3$ and $r=2$ as the features on the quadcopter are roughly in a planar configuration, and samples a random subset (``RANSAC") of the same size. The ground truth in this test was obtained from the OptiTrack system. The control of the quadcopter based on its positioning was done using a simple PID controller.

For different calculation cycles, we computed the average error through out the whole test, which consisted of roughly $4500$ frames. The results are shown in Fig.~\ref{figIoT}.

\begin{figure}[h!]
\begin{subfigure}[h]{0.49\textwidth}
\centering
\includegraphics[scale=0.6]{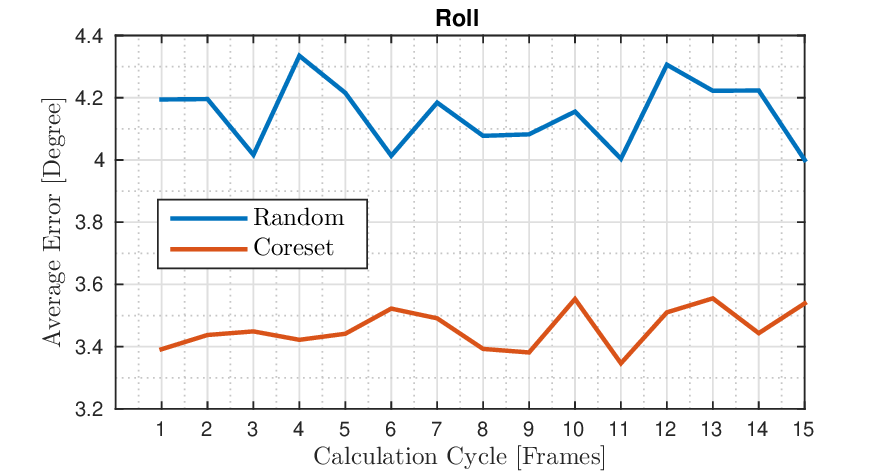}
\end{subfigure}
\begin{subfigure}[h]{0.49\textwidth}
\centering
\includegraphics[scale=0.6]{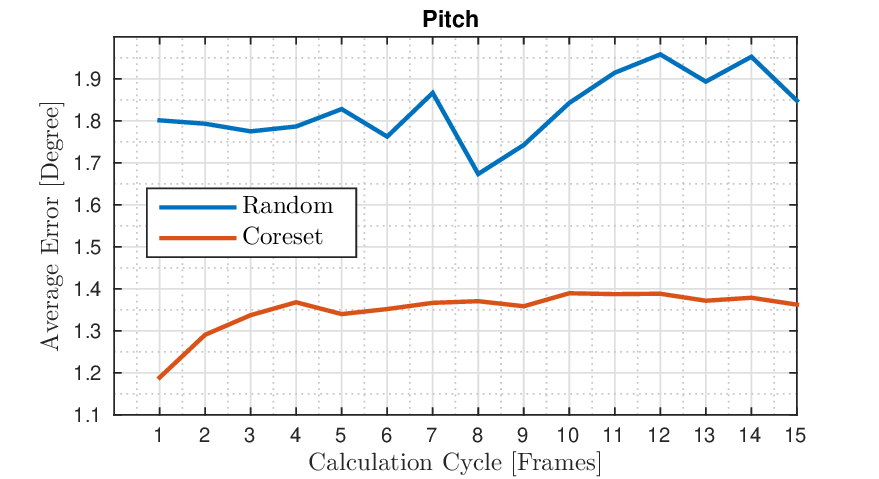}
\end{subfigure}
\begin{subfigure}[h]{0.49\textwidth}
\centering
\includegraphics[scale=0.6]{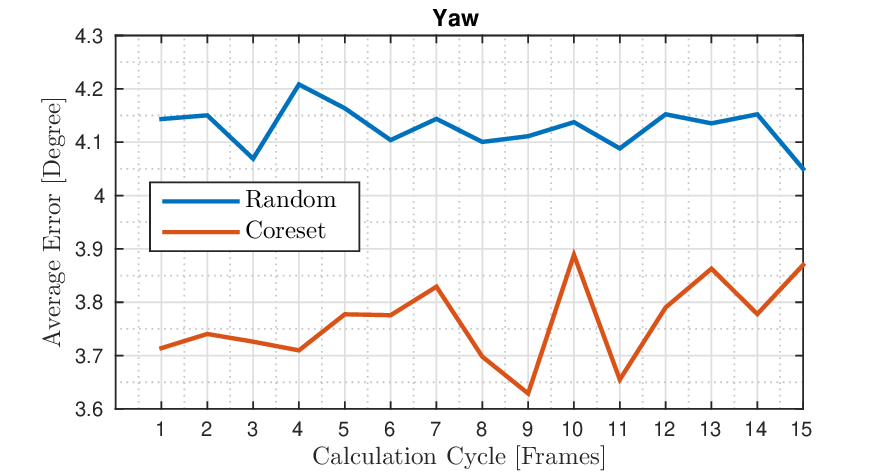}
\end{subfigure}
\caption{\small IR Tracking test: For every calculation cycle (X-axis), we compare between the core-set average error and the uniform random sampling average error. The Y-axis shows the whole test average error for each calculation cycle.\label{figIoT}}
\end{figure}

\paragraph{RGB Tracking.} To test larger set of points, we used our tracking system to track visual features (RGB images). We placed a simple planar pattern on a quadcopter, see Fig.~\ref{fig2}.
Due to the time complexity of extracting visual features, we also placed few IR reflective markers and used the OptiTrack motion capture system to perform an autonomous hover with the quadcopter, whilst two other 2D grayscale cameras mounted above the quadcopter collected and tracked visual features from the pattern using SIFT feature detector; see submitted video. The matching between the SIFT features in both images has some mismatches. This is discussed at the end of Section~\ref{sec:ourContrib}.
Given 2D coordinates of the extracted visual features from two cameras, we were able to compute the 3D location of each detected feature using triangulation. As in the IR markers test, a Kabsch-Coreset of size $5$ was computed, along side a random sample of the same size.
The quadcopter's orientation was then estimated by computing the optimal rotation matrix, using the Kabsch algorithm, on both the coreset points and the random sampled points.
The ground truth in this test was obtained using the Kabsch algorithm on all the points in the current frame.

For different calculation cycles, we computed the average error through out the whole test, which consisted of $\sim3000$ frames, as shown in Fig.~\ref{figSG}. The number of detected SIFT features in each frame was $60--100$, though most of the features did not last for more than $15$ consequent frames, therefore we tested the coreset with calculation cycles in the range $1$ to $15$. The average errors were smaller than the average errors in the previous test due to the inaccurate 3D estimation using low-cost hardware in the previous test, e.g. $\$5$ web-cams as compared to OptiTrack's $\$1000$ cameras, and due to the difference between the ground truth measurements in the two tests.

\begin{figure}[h!]
\begin{subfigure}[h]{0.49\textwidth}
\centering
\includegraphics[scale=0.6]{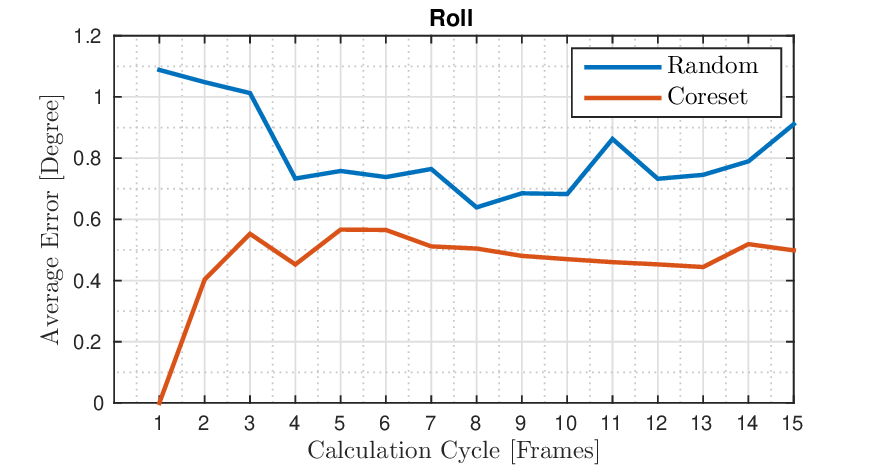}
\end{subfigure}
\begin{subfigure}[h]{0.49\textwidth}
\centering
\includegraphics[scale=0.6]{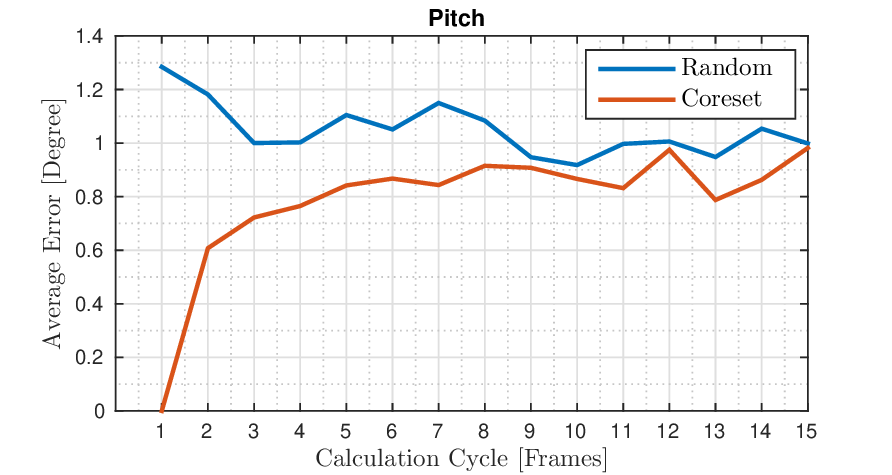}
\end{subfigure}
\begin{subfigure}[h]{0.49\textwidth}
\centering
\includegraphics[scale=0.6]{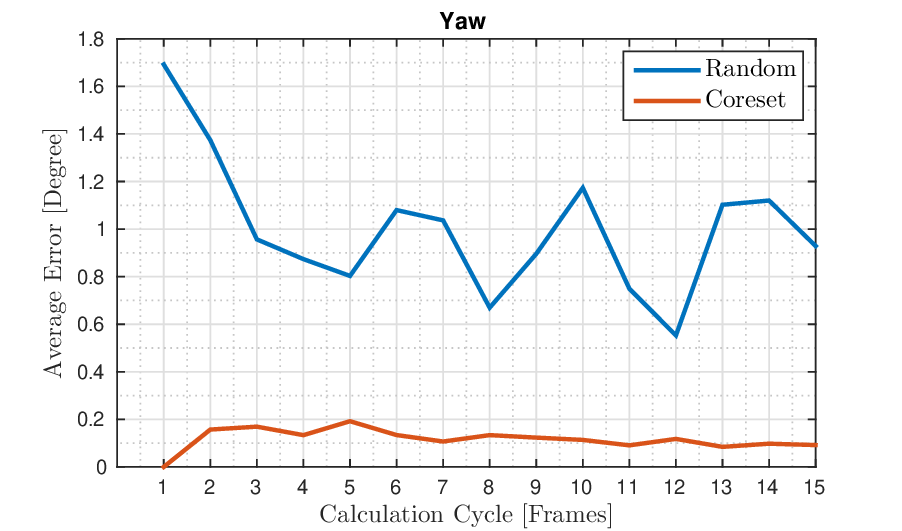}
\end{subfigure}
\caption{\small RGB Tracking test: For every calculation cycle (X-axis), we compare between the core-set average error and the uniform random sampling average error. The Y-axis shows the whole test ($3000$ frames) average error for each calculation cycle.\label{figSG}}
\end{figure}

\section{Conclusion}
We demonstrated how coresets that are usually used for solving problems in machine learning or computational geometry, can also turn theorems into real-time systems. We suggested new coresets of constant size for for kinematic data points in 3-dimensional space. This enabled us to compute the Kabsch algorithm in real-time on slow devices by running them on the coreset, while getting provably exactly the same results. In the companion video we demonstrate the first low-cost wireless tracking system that use coresets and turn a toy quadcopter into a ``Guardian Angel`` that leads guests to their desired location.

Open problems include extending our coresets for handling outliers, matching between frames, different cost functions and inputs, and multiple rigid bodies. Open-source code of our system and algorithms can be found in~\cite{opencode} . We thank Daniela Rus for suggesting the name "Guardian Angel" for our guiding system.

\bibliographystyle{SageH}
\bibliography{references}

\begin{thebibliography}{47}
\providecommand{\natexlab}[1]{#1}
\providecommand{\url}[1]{\texttt{#1}}
\providecommand{\urlprefix}{URL }
\expandafter\ifx\csname urlstyle\endcsname\relax
  \providecommand{\doi}[1]{DOI:\discretionary{}{}{}#1}\else
  \providecommand{\doi}{DOI:\discretionary{}{}{}\begingroup
  \urlstyle{rm}\Url}\fi

\bibitem[{Agarwal et~al.(2005)Agarwal, Har-Peled and Varadarajan}]{sur}
Agarwal PK, Har-Peled S and Varadarajan KR (2005) Geometric approximations via
  coresets.
\newblock \emph{Combinatorial and Computational Geometry - MSRI Publications}
  52: 1--30.

\bibitem[{Agarwal and Procopiuc(2000)}]{aga1}
Agarwal PK and Procopiuc CM (2000) Approximation algorithms for projective
  clustering.
\newblock In: \emph{Proc. 11th Annu. {ACM}-SIAM Symp. on Discrete Algorithms
  ({SODA})}.
\newblock ISBN 0-89871-453-2, pp. 538--547.

\bibitem[{Agarwal et~al.(2002)Agarwal, Procopiuc and Varadarajan}]{aga2}
Agarwal PK, Procopiuc CM and Varadarajan KR (2002) Approximation algorithms for
  k-line center.
\newblock In: \emph{Proc. 10th Annu. European Symp. on Algorithms ({ESA})},
  \emph{Lecture Notes in Computer Science}, volume 2461. Springer.
\newblock ISBN 3-540-44180-8, pp. 54--63.

\bibitem[{Aghamolaei et~al.(2015)Aghamolaei, Farhadi and
  Zarrabi-Zadeh}]{aghamolaei2015diversity}
Aghamolaei S, Farhadi M and Zarrabi-Zadeh H (2015) Diversity maximization via
  composable coresets.
\newblock In: \emph{Proceedings of the 27th Canadian Conference on
  Computational Geometry}.

\bibitem[{Ahuja et~al.(1993)Ahuja, Magnanti and Orlin}]{ahuja1993network}
Ahuja RK, Magnanti TL and Orlin JB (1993) Network flows: theory, algorithms,
  and applications .

\bibitem[{Alexandroni et~al.(2016)Alexandroni, Moreno, Sochen and
  Greenspan}]{alexandroni2016coresets}
Alexandroni G, Moreno GZ, Sochen N and Greenspan H (2016) Coresets versus
  clustering: comparison of methods for redundancy reduction in very large
  white matter fiber sets.
\newblock In: \emph{SPIE Medical Imaging}. International Society for Optics and
  Photonics, pp. 97840A--97840A.

\bibitem[{Bachem et~al.(2016)Bachem, Lucic, Hassani and
  Krause}]{bachem2016approximate}
Bachem O, Lucic M, Hassani SH and Krause A (2016) Approximate k-means++ in
  sublinear time.
\newblock In: \emph{Conference on Artificial Intelligence (AAAI)}.

\bibitem[{Bachem et~al.(2015)Bachem, Lucic and Krause}]{bachem2015coresets}
Bachem O, Lucic M and Krause A (2015) Coresets for nonparametric
  estimation—the case of dp-means.
\newblock In: \emph{International Conference on Machine Learning (ICML)}.

\bibitem[{Besl and McKay(1992)}]{besl1992method}
Besl PJ and McKay ND (1992) Method for registration of 3-d shapes.
\newblock In: \emph{Robotics-DL tentative}. International Society for Optics
  and Photonics, pp. 586--606.

\bibitem[{Buriol et~al.(2007)Buriol, Frahling, Leonardi and
  Sohler}]{BuriolFLS07}
Buriol LS, Frahling G, Leonardi S and Sohler C (2007) Estimating clustering
  indexes in data streams.
\newblock In: \emph{Proc. 15th Annu. European Symp. on Algorithms (ESA)},
  \emph{Lecture Notes in Computer Science}, volume 4698. Springer.
\newblock ISBN 978-3-540-75519-7, pp. 618--632.
\newblock \urlprefix\url{http://dx.doi.org/10.1007/978-3-540-75520-3_55}.

\bibitem[{Carath{\'e}odory(1911)}]{cara}
Carath{\'e}odory C (1911) {\"U}ber den variabilit{\"a}tsbereich der
  fourier’schen konstanten von positiven harmonischen funktionen.
\newblock \emph{Rendiconti del Circolo Matematico di Palermo (1884-1940)}
  32(1): 193--217.

\bibitem[{Czumaj et~al.(2005)Czumaj, Erg{\"u}n, Fortnow, Magen, Newman,
  Rubinfeld and Sohler}]{czumaj2005approximating}
Czumaj A, Erg{\"u}n F, Fortnow L, Magen A, Newman I, Rubinfeld R and Sohler C
  (2005) Approximating the weight of the euclidean minimum spanning tree in
  sublinear time.
\newblock \emph{SIAM Journal on Computing} 35(1): 91--109.

\bibitem[{Czumaj and Sohler(2007)}]{CzuSoh07a}
Czumaj A and Sohler C (2007) Sublinear-time approximation algorithms for
  clustering via random sampling.
\newblock \emph{Random Struct. Algorithms (RSA)} 30(1-2): 226--256.
\newblock \urlprefix\url{http://dx.doi.org/10.1002/rsa.20157}.

\bibitem[{de~Carli~Silva et~al.(2011)de~Carli~Silva, Harvey and Sato}]{summ}
de~Carli~Silva MK, Harvey NJA and Sato CM (2011) Sparse sums of positive
  semidefinite matrices.
\newblock \emph{CoRR} abs/1107.0088.
\newblock \urlprefix\url{http://arxiv.org/abs/1107.0088}.

\bibitem[{Feigin et~al.(2011)Feigin, Feldman and Sochen}]{feigin2011high}
Feigin M, Feldman D and Sochen N (2011) From high definition image to low space
  optimization.
\newblock In: \emph{International Conference on Scale Space and Variational
  Methods in Computer Vision}. Springer, pp. 459--470.

\bibitem[{Feldman et~al.(2016{\natexlab{a}})Feldman, , Xian and
  Rus}]{rusprivate}
Feldman D, , Xian C and Rus D (2016{\natexlab{a}}) Private coresets for
  high-dimensional spaces, submitted.
\newblock Technical report.

\bibitem[{Feldman et~al.(2011)Feldman, Faulkner and
  Krause}]{feldman2011scalable}
Feldman D, Faulkner M and Krause A (2011) Scalable training of mixture models
  via coresets.
\newblock In: \emph{Advances in neural information processing systems (NIPS)}.
  pp. 2142--2150.

\bibitem[{Feldman et~al.(2013{\natexlab{a}})Feldman, Feigin and
  Sochen}]{feldman2013learning}
Feldman D, Feigin M and Sochen N (2013{\natexlab{a}}) Learning big (image) data
  via coresets for dictionaries.
\newblock \emph{Journal of mathematical imaging and vision} 46(3): 276--291.

\bibitem[{Feldman et~al.(2007)Feldman, Monemizadeh and
  Sohler}]{feldman2007ptas}
Feldman D, Monemizadeh M and Sohler C (2007) A {PTAS} for {\em k}-means
  clustering based on weak coresets.
\newblock In: \emph{Proc. 23rd {ACM} Symp. on Computational Geometry (SoCG)}.

\bibitem[{Feldman et~al.(2013{\natexlab{b}})Feldman, Sugaya, Sung and
  Rus}]{feldman2013idiary}
Feldman D, Sugaya A, Sung C and Rus D (2013{\natexlab{b}}) idiary: from gps
  signals to a text-searchable diary.
\newblock In: \emph{Proceedings of the 11th ACM Conference on Embedded
  Networked Sensor Systems}. ACM, p.~6.

\bibitem[{Feldman et~al.(2016{\natexlab{b}})Feldman, Volkov and
  Rus}]{feldmanmik}
Feldman D, Volkov M and Rus D (2016{\natexlab{b}}) Dimensionality reduction of
  massive sparse datasets using coresets.
\newblock In: \emph{Advances in neural information processing systems (NIPS)}.

\bibitem[{Frahling et~al.(2008)Frahling, Indyk and Sohler}]{FrahlingIS08}
Frahling G, Indyk P and Sohler C (2008) Sampling in dynamic data streams and
  applications.
\newblock \emph{Int. J. Comput. Geometry Appl.} 18(1/2): 3--28.
\newblock \urlprefix\url{http://dx.doi.org/10.1142/S0218195908002520}.

\bibitem[{Frahling and Sohler(2005)}]{FraSoh05}
Frahling G and Sohler C (2005) Coresets in dynamic geometric data streams.
\newblock In: \emph{Proc. 37th Annu. {ACM} Symp. on Theory of Computing
  (STOC)}. pp. 209--217.

\bibitem[{Friedman et~al.(1977)Friedman, Bentley and
  Finkel}]{friedman1977algorithm}
Friedman JH, Bentley JL and Finkel RA (1977) An algorithm for finding best
  matches in logarithmic expected time.
\newblock \emph{ACM Transactions on Mathematical Software (TOMS)} 3(3):
  209--226.

\bibitem[{Har-Peled(2004)}]{HP01}
Har-Peled S (2004) Clustering motion.
\newblock \emph{Discrete Comput. Geom.} 31(4): 545--565.
\newblock \urlprefix\url{http://dx.doi.org/10.1007/s00454-004-2822-7}.

\bibitem[{Huggins et~al.(2016)Huggins, Campbell and
  Broderick}]{huggins2016coresets}
Huggins JH, Campbell T and Broderick T (2016) Coresets for scalable bayesian
  logistic regression.
\newblock \emph{arXiv preprint arXiv:1605.06423} .

\bibitem[{Indyk et~al.(2014)Indyk, Mahabadi, Mahdian and
  Mirrokni}]{indyk2014composable}
Indyk P, Mahabadi S, Mahdian M and Mirrokni VS (2014) Composable core-sets for
  diversity and coverage maximization.
\newblock In: \emph{Proceedings of the 33rd ACM SIGMOD-SIGACT-SIGART symposium
  on Principles of database systems}. ACM, pp. 100--108.

\bibitem[{Kabsch(1976)}]{kabsch1976solution}
Kabsch W (1976) A solution for the best rotation to relate two sets of vectors.
\newblock \emph{Acta Crystallographica Section A: Crystal Physics, Diffraction,
  Theoretical and General Crystallography} 32(5): 922--923.

\bibitem[{Kjer and Wilm(2010)}]{kjer2010evaluation}
Kjer HM and Wilm J (2010) \emph{Evaluation of surface registration algorithms
  for PET motion correction}.
\newblock PhD Thesis, Technical University of Denmark, DTU, DK-2800 Kgs.
  Lyngby, Denmark.

\bibitem[{Lepetit et~al.(2009)Lepetit, Moreno-Noguer and Fua}]{lepetit2009epnp}
Lepetit V, Moreno-Noguer F and Fua P (2009) Epnp: An accurate o(n) solution to
  the pnp problem.
\newblock \emph{International journal of computer vision} 81(2): 155--166.

\bibitem[{Lucic et~al.(2016)Lucic, Bachem and Krause}]{lucic2016strong}
Lucic M, Bachem O and Krause A (2016) Strong coresets for hard and soft bregman
  clustering with applications to exponential family mixtures.
\newblock In: \emph{Proceedings of the 19th International Conference on
  Artificial Intelligence and Statistics}. pp. 1--9.

\bibitem[{Lucic et~al.(2015)Lucic, Ohannessian, Karbasi and
  Krause}]{lucic2015tradeoffs}
Lucic M, Ohannessian MI, Karbasi A and Krause A (2015) Tradeoffs for space,
  time, data and risk in unsupervised learning.
\newblock In: \emph{AISTATS}.

\bibitem[{Mirrokni and Zadimoghaddam(2015)}]{mirrokni2015randomized}
Mirrokni V and Zadimoghaddam M (2015) Randomized composable core-sets for
  distributed submodular maximization.
\newblock In: \emph{Proceedings of the Forty-Seventh Annual ACM on Symposium on
  Theory of Computing}. ACM, pp. 153--162.

\bibitem[{Nasser et~al.(2017{\natexlab{a}})Nasser, Jubran and
  Feldman}]{coresetVideo17}
Nasser S, Jubran I and Feldman D (2017{\natexlab{a}}) Coreset video.
\newblock \url{https://vimeo.com/244656298}.

\bibitem[{Nasser et~al.(2017{\natexlab{b}})Nasser, Jubran and
  Feldman}]{opencode}
Nasser S, Jubran I and Feldman D (2017{\natexlab{b}}) Open code for the
  algorithms and system in this paper.
\newblock The authors committ to publish upon acceptance of this paper.

\bibitem[{Phillips(2016)}]{Phillips16}
Phillips JM (2016) Coresets and sketches, near-final version of chapter 49 in
  handbook on discrete and computational geometry, 3rd edition.
\newblock \emph{CoRR} abs/1601.00617.
\newblock \urlprefix\url{http://arxiv.org/abs/1601.00617}.

\bibitem[{Pulli and Shapiro(2000)}]{pulli2000surface}
Pulli K and Shapiro LG (2000) Surface reconstruction and display from range and
  color data.
\newblock \emph{Graphical Models} 62(3): 165--201.

\bibitem[{Reddi et~al.(2015)Reddi, P{\'o}czos and
  Smola}]{reddi2015communication}
Reddi SJ, P{\'o}czos B and Smola A (2015) Communication efficient coresets for
  empirical loss minimization.
\newblock In: \emph{Conference on Uncertainty in Artificial Intelligence
  (UAI)}.

\bibitem[{Rosman et~al.(2014)Rosman, Volkov, Feldman, Fisher~III and
  Rus}]{rosman2014coresets}
Rosman G, Volkov M, Feldman D, Fisher~III JW and Rus D (2014) Coresets for
  k-segmentation of streaming data.
\newblock In: \emph{Advances in Neural Information Processing Systems (NIPS)}.
  pp. 559--567.

\bibitem[{Senseable City~Lab(2016)}]{skycall}
Senseable City~Lab M (2016) Skycall video.
\newblock \url{https://www.youtube.com/watch?v=mB9NfEJ0ZVs}.

\bibitem[{Stanway and Kinsey(2015)}]{stanway2015rotation}
Stanway MJ and Kinsey JC (2015) Rotation identification in geometric algebra:
  Theory and application to the navigation of underwater robots in the field.
\newblock \emph{Journal of Field Robotics} .

\bibitem[{Sung et~al.(2012)Sung, Feldman and Rus}]{sung2012trajectory}
Sung C, Feldman D and Rus D (2012) Trajectory clustering for motion prediction.
\newblock In: \emph{2012 IEEE/RSJ International Conference on Intelligent
  Robots and Systems}. IEEE, pp. 1547--1552.

\bibitem[{Tsang et~al.(2005)Tsang, Kwok and Cheung}]{tsang2005core}
Tsang IW, Kwok JT and Cheung PM (2005) Core vector machines: Fast svm training
  on very large data sets.
\newblock \emph{Journal of Machine Learning Research} 6(Apr): 363--392.

\bibitem[{Valencia and Vargas(2015)}]{bipmatching}
Valencia CE and Vargas MC (2015) Optimum matchings in weighted bipartite
  graphs.
\newblock \emph{Boletín de la Sociedad Matemática Mexicana} .

\bibitem[{Wahba(1965)}]{wahba65}
Wahba G (1965) A least squares estimate of satellite attitude.
\newblock \emph{SIAM review} 7(3): 409--409.

\bibitem[{Wang and Sun(2015)}]{wang2015comparisons}
Wang L and Sun X (2015) Comparisons of iterative closest point algorithms.
\newblock In: \emph{Ubiquitous Computing Application and Wireless Sensor}.
  Springer, pp. 649--655.

\bibitem[{Zhang(1994)}]{zhang1994iterative}
Zhang Z (1994) Iterative point matching for registration of free-form curves
  and surfaces.
\newblock \emph{International journal of computer vision} 13(2): 119--152.

\end{thebibliography}

\end{document}